\pdfoutput=1

\documentclass[11pt]{article}

\usepackage{emnlp2023}

\usepackage{times}
\usepackage{latexsym}
\usepackage[htt]{hyphenat}
\usepackage{amssymb}
\usepackage{amsmath}
\usepackage{booktabs}
\usepackage{multirow}
\usepackage{graphicx}
\usepackage{enumitem}
\usepackage{algorithm}
\usepackage{algpseudocode}
\usepackage{amsthm}
\newtheorem{theorem}{Theorem}
\newtheorem{lemma}{Lemma}
\usepackage[T1]{fontenc}

\usepackage[utf8]{inputenc}

\usepackage{microtype}

\usepackage{inconsolata}

\title{Multi-Expert Conformal Risk Control for Pairwise LLM Judging \\ in Open-Ended Dialogue}

\author{
  Ming Cheng \quad
  Yusheng Dai \quad
  Qiuhong Ke$^{\dagger}$ \quad
  Zhaolin Chen \quad
  Lizhen Qu$^{\dagger}$ \\
  Monash University \\
  \texttt{\{ming.cheng, yusheng.dai, qiuhong.ke, zhaolin.chen, lizhen.qu\}@monash.edu}
}

\begin{document}
\maketitle
\renewcommand{\thefootnote}{\fnsymbol{footnote}}
\footnotetext[2]{Corresponding Authors.}
\renewcommand{\thefootnote}{\arabic{footnote}}
\setcounter{footnote}{0}
\begin{abstract}



In this paper, we explore multi-expert Conformal Risk Control (CRC) algorithms for pairwise LLM-as-a-Judge evaluation in open-ended dialogue. Our core insight is that multi-expert aggregation offers a complementary remedy to CRC: whereas CRC controls risk at the decision threshold through abstention, aggregation sanitizes the scoring function at its source. Guided by this, we first design two multi-expert CRC methods: \textit{Score Averaging} and \textit{Decision Voting}, which aggregate at the score and decision levels, respectively. While both strategies outperform single-expert methods on homogeneous expert panels, on heterogeneous LLM judges they remain risk-valid but recover only limited coverage, because a uniform threshold cannot match the experts' distinct scoring scales. To resolve this issue, we further propose \textit{Marginal-Calibrated Conformal Consensus} (MC$^3$): it captures distinct per-expert scales via initial threshold ratios, while jointly tuning a unified decision function~$C_t(x)$ applied identically in both calibration and test, thereby preserving exchangeability. To evaluate our framework, we construct \textit{Panel}, a 1{,}800-pair human pairwise-preference benchmark for open-ended dialogue. It is built on responses generated by four open-weight LLMs over dialogue contexts from three domains (ESConv, MSC, DREAM), with full logit access. In experiments, we find that both Score Averaging and Decision Voting substantially improve accuracy and acceptance rate on homogeneous panels. Notably, MC$^3$ extends these gains to heterogeneous panels by accommodating distinct per-expert scoring scales across all three datasets.

\end{abstract}


\section{Introduction}
\label{sec:intro}

\begin{figure*}[t]
\centering
\includegraphics[width=0.95\textwidth]{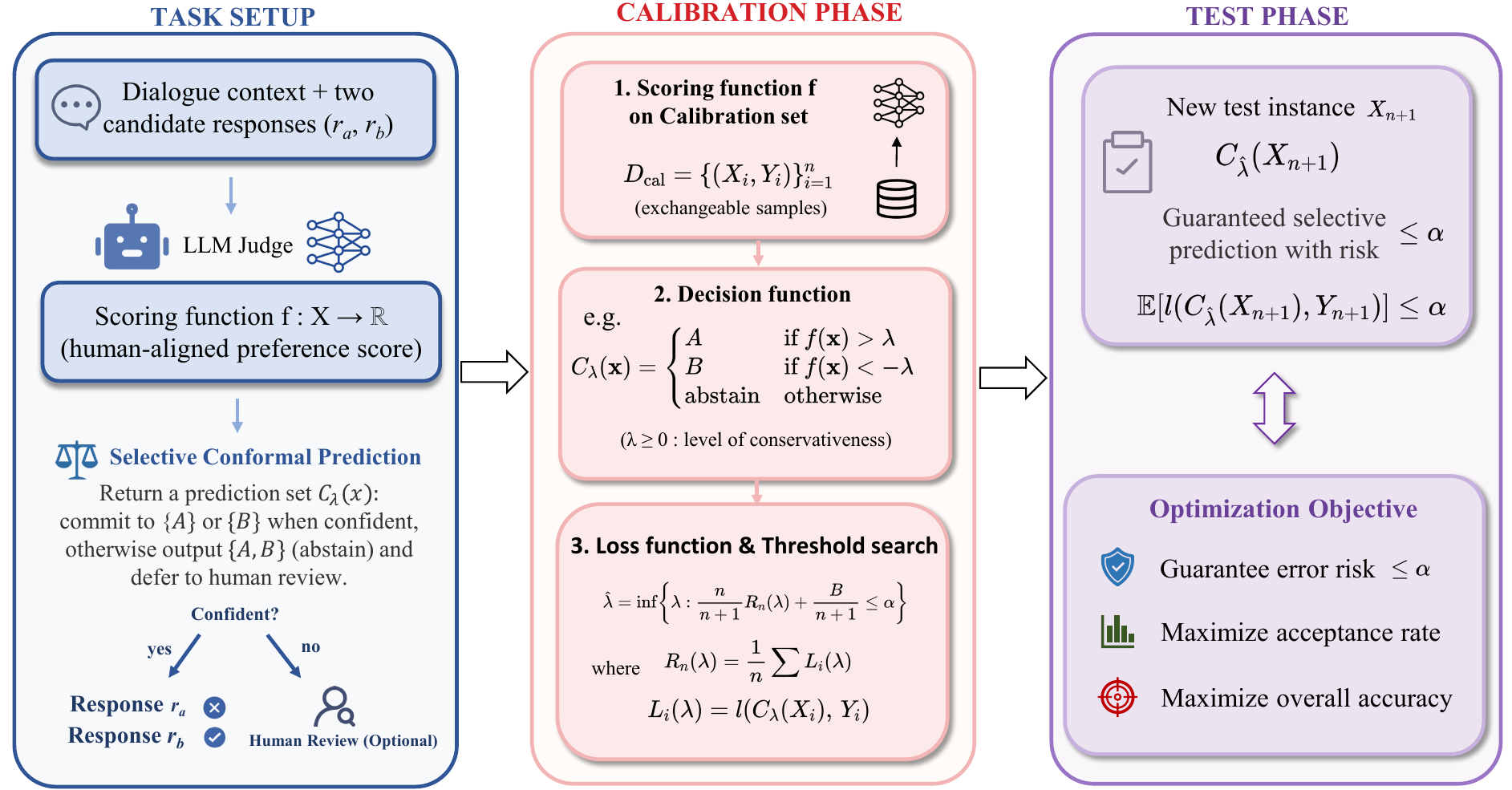}
\caption{Overview of the selective prediction pipeline under Conformal Risk Control (CRC). \textbf{Left}: the task setup defines evaluation instances and the calibration set. \textbf{Middle}: the calibration phase constructs a scoring function~$f$ and a parameterized decision function~$C_\lambda$, then searches for the smallest threshold~$\hat{\lambda}$ that keeps the empirical risk below~$\alpha$. \textbf{Right}: the test phase deploys~$C_{\hat{\lambda}}$ on new instances with a distribution-free guarantee, while the optimization objective is to maximize acceptance rate and accuracy under the risk constraint.}
\label{fig:overview}
\end{figure*}

LLM-as-a-judge has become the standard for evaluating open-ended dialogue, replacing costly human preference labels with model-based pairwise comparisons~\cite{zheng2023judging,kim2024prometheus}. Yet, in high-stakes domains such as mental-health support, relying on automated judges carries significant risk: even minor misjudgments, driven by overconfidence or systematic biases, can have severe consequences~\cite{wang2024fair,koo2024cobbler}. Consequently, the safe and reliable use of these LLM evaluators in such settings demands explicit mechanisms for strict risk control.

To operationalize this risk control, we formulate LLM-based pairwise evaluation as a \textit{selective prediction} problem~\cite{jung2024trust,chen2023adaptation} under the Conformal Risk Control (CRC) framework~\cite{crc} (detailed in Section~\ref{sec:formulation}). Instead of forcing a verdict, the LLM judge commits to a preference only when sufficiently confident; otherwise, it abstains and returns the full candidate set (optionally deferring the pair to human review). By calibrating the abstention threshold $\lambda$ on the calibration set, CRC provides a formal guarantee that the error risk on the unseen test set is bounded by a user-specified level~$\alpha$~\cite{crc}. Such a guarantee is finite-sample and distribution-free, requiring only two conditions: the exchangeability of the calibration and test pairs and the non-increasing monotonicity of the loss function in $\lambda$. Fundamentally, this selective paradigm achieves risk control by trading coverage for accuracy, abstaining whenever the judge is unsure. Consequently, preserving the practical utility of this approach requires a discriminative, human-aligned scoring function to maximize evaluation coverage while strictly adhering to the specified risk bound.

However, prior work on CRC for LLM evaluation commonly relies on a single-expert paradigm, which is empirically miscalibrated for pairwise comparisons~\cite{jung2024trust,scope}. While advanced models partially mitigate shallow artifacts such as position bias~\cite{zheng2023judging} and self-enhancement~\cite{panickssery2024llm} (Table~\ref{tab:main_results}), deeper model-specific biases persist as a fundamental bottleneck~\cite{koo2024cobbler,dorner2024limits}: intrinsic preferences shaped by training data, prompt sensitivities, and domain-specific competence gaps remain embedded in any individual evaluator~\cite{feuer2025style,stureborg2024inconsistent}. Lacking explicit ground truth, open-ended dialogues severely amplify these idiosyncrasies, exposing a structural limitation unsolvable by any solitary judge. Judge ensembles offer a natural solution~\cite{verga2024poll}: whether implemented through multi-agent debate or multi-model panels, they build a robust consensus that yields more discriminative and human-aligned evaluation~\cite{chan2024chateval,autoarena2025,verga2024poll,zhao2025lmc,qian2026jury}. Yet how multi-expert mechanisms can be integrated into CRC, and how much coverage they could recover at a fixed risk level, therefore remain largely unexplored.

In this paper, we present the first formal multi-expert CRC framework for pairwise LLM-as-a-Judge evaluation in open-ended dialogue settings. Our framework is built on the core insight that multi-expert aggregation offers a complementary remedy to CRC: whereas CRC controls risk at the decision threshold through abstention, multi-expert aggregation sanitizes the scoring function at its source. Guided by this, we design two CRC-adapted multi-expert strategies: Score Averaging, which aggregates at the scoring function~$f$, and Decision Voting, which aggregates at the decision function~$C_\lambda$. We study these strategies on two kinds of expert panels: \textit{homogeneous} (the same model under varied prompts or random seeds) and \textit{heterogeneous} (different models evaluating independently). While both clearly improve over single-expert approaches in the homogeneous setting, in the heterogeneous setting they remain risk-valid yet recover less coverage, because a single shared threshold cannot match the experts' distinct scoring scales and silences the narrow-range ones (\S\ref{sec:avg_vote}).

To further address the heterogeneous setting, we propose \textit{Marginal-Calibrated Conformal Consensus} (MC$^3$). MC$^3$ uses per-expert independent calibration as initialization to estimate the ratio between experts' thresholds, then performs joint threshold searching over a single global scalar~$t$ that scales all per-expert thresholds proportionally as $\lambda_j(t) = t \cdot \lambda_j^{(0)}$. By defining a unified joint vote decision function~$C_t(x)$ used identically in both calibration and test phases, MC$^3$ preserves exchangeability and the finite-sample, distribution-free risk guarantee of CRC (Appendix~\ref{sec:appendix_proofs}). Empirically, by retaining the multi-expert ensemble while explicitly accounting for the heterogeneity across experts, MC$^3$ attains the highest acceptance rate among CRC-compliant methods on all three datasets, at accuracy comparable to Decision Voting and while preserving the formal risk guarantee throughout.

Existing dialogue benchmarks either annotate within-model preferences (e.g., ESC-Pro~\cite{Zhao2025}, EmPO~\cite{Sotolar2024}) or rely on closed-source LLMs without logit access (e.g., HEART~\cite{heart2025}, o2mDial~\cite{Lee2025}), leaving white-box CRC evaluation unsupported. To fill this gap, we construct \textsc{Panel}: 1{,}800 exhaustive pairwise comparisons between four open-weight LLMs across three open-ended dialogue domains (ESConv~\cite{esconv}, MSC~\cite{msc}, DREAM~\cite{dream}), with human preference labels under a multi-dimensional rubric and full logit access.

We make three key contributions:
\begin{itemize}[leftmargin=*, itemsep=2pt, topsep=4pt]
    \item Through a systematic empirical study, we identify the optimal conformity score for pairwise LLM-as-a-Judge CRC in the open-ended dialogue settings, filling a gap in prior work.

    \item Based on this, we propose the first formal multi-expert CRC framework, with \textit{Score Averaging} and \textit{Decision Voting} for homogeneous expert settings, and MC$^3$ for heterogeneous expert settings.

    \item We construct \textsc{Panel}, on which our framework substantially improves accuracy and acceptance rate over single-expert baselines while maintaining the target risk level.
\end{itemize}

\begin{figure*}[t]
\centering
\includegraphics[width=1.0\textwidth]{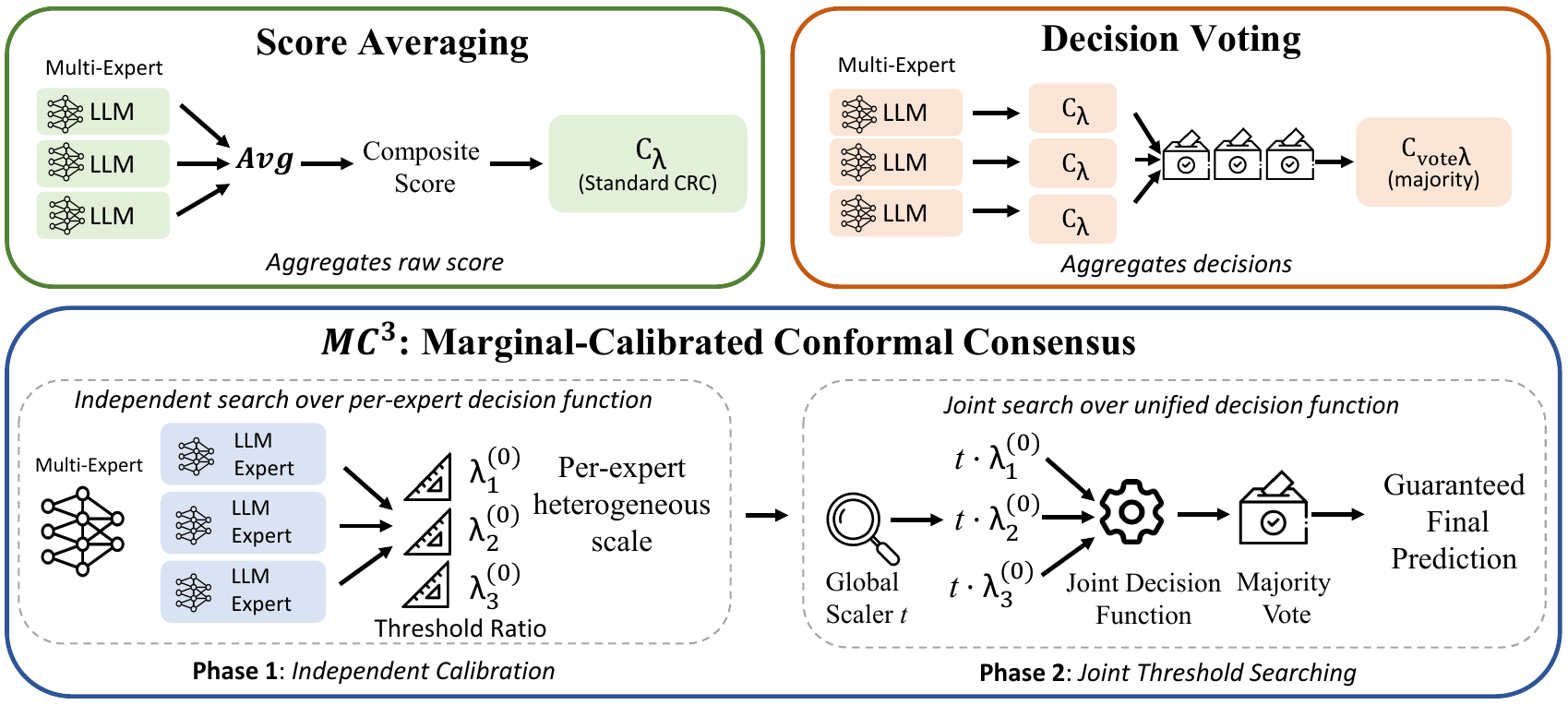}
\caption{Overview of the three multi-expert CRC strategies. \textbf{Score Averaging} (top left) aggregates per-expert scores into a composite $f_{\mathrm{avg}}$. \textbf{Decision Voting} (top right) lets each expert make an independent decision $C_\lambda^{(j)}$ with a shared threshold~$\lambda$, then aggregates via majority vote. \textbf{MC$^3$} (bottom) first initializes per-expert threshold ratios $\lambda_j^{(0)}$ via independent calibration (Phase~1), then scales them proportionally by a global scalar~$t$ and performs joint threshold searching over a unified decision function $C_t(x)$ (Phase~2), simultaneously preserving per-expert heterogeneity and the exchangeability required by CRC.}
\label{fig:method}
\end{figure*}

\section{Problem Formulation}
\label{sec:formulation}

\paragraph{Overview.}
As illustrated in Figure~\ref{fig:overview}, we formulate LLM-based pairwise evaluation as a \textit{selective prediction} problem under the Conformal Risk Control (CRC) framework~\cite{crc}. Let $\mathcal{X}$ denote the space of evaluation instances, where each $x \in \mathcal{X}$ comprises a dialogue context paired with two candidate responses $(r_a, r_b)$, and let $\mathcal{D}_{\mathrm{cal}} = \{(X_i, Y_i)\}_{i=1}^n$ denote an exchangeable calibration set of instances with human preference labels. Given an LLM judge that compares such pairs, our system performs selective prediction by returning a prediction set $C_\lambda(x) \subseteq \mathcal{Y}$. A singleton output, $C_\lambda(x) = \{A\}$ or $\{B\}$, indicates that the judge is confident enough to commit to a preference ($r_a$ or $r_b$, respectively). Conversely, when uncertain, the judge outputs the full set $C_\lambda(x) = \{A, B\}$, effectively \textit{abstaining} and deferring the instance to human evaluation. The CRC pipeline consists of (i) a \textit{calibration phase} that uses $\mathcal{D}_{\mathrm{cal}}$ to search a threshold~$\hat{\lambda}$, and (ii) a \textit{test phase} that deploys~$\hat{\lambda}$ on unseen instances with the formal risk guarantee. Our objective is to maximize both the \textit{accuracy}, which measures the judge's alignment with human preferences over all pairs regardless of acceptance, and the \textit{acceptance rate}, which represents the overall coverage of accepted predictions across all samples.

\paragraph{Scoring Function and Decision Function.}
A scoring function $f\colon \mathcal{X} \to \mathbb{R}$ assigns each instance a real-valued preference score reflecting the confidence of the LLM judge in preferring one response. A threshold parameter $\lambda \geq 0$ then defines the prediction set:
\begin{equation}
    C_\lambda(x) = \begin{cases}
        \{A\} & \text{if } f(x) > \lambda \\
        \{B\} & \text{if } f(x) < -\lambda \\
        \{A, B\} & \text{otherwise \textup{(abstain)}}
    \end{cases}
    \label{eq:decision}
\end{equation}
Here $\lambda$ controls conservativeness: larger $\lambda$ raises the acceptance bar, increasing abstentions at the cost of coverage.

\paragraph{Loss Function and Threshold Searching.}
For each calibration instance $(X_i, Y_i)$, the loss is the \textit{miscoverage} of the prediction set:
\begin{equation}
    L_i(\lambda) = \mathbf{1}\!\left[\,Y_i \notin C_\lambda(X_i)\,\right]
    \label{eq:loss}
\end{equation}
which is non-increasing in~$\lambda$ (Lemmas~\ref{lem:perexpert} and~\ref{lem:lossmono}): larger $\lambda$ accepts fewer instances and thus incurs fewer errors. CRC then selects $\hat{\lambda}$ as:
\begin{equation}
    \hat{\lambda} = \inf\left\{\lambda \,:\, \frac{n}{n{+}1}\, R_n(\lambda) + \frac{B}{n{+}1} \leq \alpha \right\}
    \label{eq:calibration}
\end{equation}
where $R_n(\lambda) = \frac{1}{n}\sum_{i=1}^n L_i(\lambda)$ is the empirical risk and $B$ is a uniform upper bound on the loss ($L_i(\lambda) \leq B$ for every~$\lambda$). The term $B/(n{+}1)$ bounds the worst-case contribution of the single unseen test point to the risk, which is what lets the guarantee hold at any finite~$n$. Since the loss in Eq.~\ref{eq:loss} is the $0/1$ miscoverage indicator, it deterministically takes values in $\{0, 1\}$: $L_i(\lambda) = 1$ only when the judge commits to a single response that disagrees with the human label, whereas an abstention returns $\{A, B\}$ which trivially contains $Y_i$ and hence $L_i = 0$. Consequently, $B = 1$ is the tightest valid upper bound throughout. The $\inf$ yields the smallest qualifying $\lambda$, maximizing acceptance under the risk constraint.

\paragraph{Deployment Guarantee.}
For a new test instance $X_{n+1}$ drawn exchangeably with the calibration data, the calibrated predictor $C_{\hat{\lambda}}$ satisfies:
\begin{equation}
    \mathbb{E}\bigl[\ell(C_{\hat{\lambda}}(X_{n+1}), Y_{n+1})\bigr] \leq \alpha
    \label{eq:guarantee}
\end{equation}
This guarantee is distribution-free, model-agnostic, and finite-sample, requiring only exchangeability and loss monotonicity in $\lambda$.


\section{Multi-Expert Conformal Consensus}
\label{sec:multi_expert}

Surface biases such as position bias and self-enhancement fade as models scale up (Table~\ref{tab:main_results}), but model-specific tendencies, distinct scoring scales, domain preferences, and prompt sensitivities, persist regardless of scale~\cite{koo2024cobbler}. Multi-expert aggregation reduces the variance from these residual biases, which single-model improvements cannot. Yet no prior work has studied multi-expert CRC with formal risk guarantees, a gap we address in this section.


\subsection{Per-Expert Base Score Selection}
\label{sec:base_score}

We first conduct an extensive empirical study of base scoring functions for CRC in open-ended dialogue (see \S\ref{sec:experiments} for full results). Among all candidates, \textit{logit-based pairwise preference probability} emerges as the optimal single-expert conformity score. Let $z_A^{(j)}, z_B^{(j)}$ denote the judge's logits at the first generated token position for tokens \texttt{A}, \texttt{B} on instance~$x$. Each expert's scoring function is a binary softmax over these two logits:
\begin{equation}
    f_j(x) = \frac{\exp(z_A^{(j)})}{\exp(z_A^{(j)}) + \exp(z_B^{(j)})} - 0.5
    \label{eq:base_score}
\end{equation}
Being continuous, this score allows the threshold to be searched at fine granularity, and it satisfies the monotonicity required by CRC. It also achieves the strongest alignment with human preferences among all candidates (Table~\ref{tab:main_results}). We therefore adopt $f_j(x)$ as the default per-expert base score for multi-expert CRC throughout this work.

\subsection{Score Averaging and Decision Voting}
\label{sec:avg_vote}

Within the CRC pipeline of \S\ref{sec:formulation}, multiple experts can be combined at two natural points: the scoring function~$f$ and the decision function~$C_\lambda$. We design one CRC-adapted strategy for each point, with all experts sharing a single threshold~$\lambda$ for simplicity:

\paragraph{Score Averaging} modifies the scoring function~$f$ while leaving the pipeline unchanged. The $K$ per-expert scores are averaged into a composite:
\begin{equation}
    f_{\mathrm{avg}}(x) = \frac{1}{K}\sum_{j=1}^{K} f_j(x)
    \label{eq:avg_score}
\end{equation}
and standard CRC is applied to $f_{\mathrm{avg}}$: decision function $C_\lambda$, loss $L_i(\lambda)$, and threshold searching all remain identical to the single-expert case.

\paragraph{Decision Voting} retains each per-expert scoring function $f_j(x)$ but modifies the decision function~$C_\lambda$. All experts share a unified threshold~$\lambda$; each independently produces a per-expert decision $C_\lambda^{(j)}(x)$ following Eq.~\ref{eq:decision}, and a majority vote aggregates them into a joint decision function:
\begin{equation}
    C_\lambda^{\mathrm{vote}}(x) = \mathrm{MajVote}\!\bigl(C_\lambda^{(1)}(x),\;\ldots,\;C_\lambda^{(K)}(x)\bigr)
    \label{eq:decision_voting}
\end{equation}
The system accepts a prediction when a strict majority of experts agree on a direction, and abstains otherwise.

\paragraph{Adaptation on Homogeneous and Heterogeneous multi-experts.}
When all $K$ experts are instances of the same model under varied prompts or seeds (\textit{homogeneous} setting), every $f_j$ shares the same scoring scale by construction. Score Averaging and Decision Voting therefore preserve the underlying distribution, and a shared~$\lambda$ treats all experts fairly: both strategies maintain the risk guarantee (Appendix~\ref{sec:appendix_avg} and~\ref{sec:appendix_dvote}) and substantially improve accuracy over a single expert (Table~\ref{tab:multi_expert}). When experts come from different models (\textit{heterogeneous} setting), their scoring distributions become incompatible. When experts operate at different scoring scales (e.g., $\lambda \approx 0.05$ for a 12B model vs.\ $\lambda \approx 0.20$ for an 8B model), a single shared threshold treats them asymmetrically: the narrow-scale expert never crosses the threshold and is effectively \textit{silenced}, always casting an abstain vote. The system thereby degrades from a multi-expert ensemble into a single-expert one, since only the wide-scale expert can ever contribute a non-abstaining vote. \textit{Score Averaging} compounds this by first collapsing the incompatible $f_j$ distributions into a composite whose scale matches no individual expert; per-expert calibration (e.g., Platt scaling~\cite{platt1999probabilistic}) could address this but at prohibitive complexity. Neither strategy resolves the underlying issue, so both attain low acceptance rates and fail to fully exploit multi-expert benefits.

\section{\fontsize{10.5pt}{12pt}\selectfont Marginal-Calibrated Conformal Consensus}
\label{sec:mc3}

\subsection{Marginal Calibration Initialization}
\label{sec:init}

A natural approach to this is \textit{Test-time Voting}: assign each expert~$j$ an independent threshold via the per-expert analogue of Eq.~\ref{eq:calibration}:
\begin{equation}
    \lambda_j^{(0)} = \inf\Bigl\{\lambda : R_n^{(j)}(\lambda) \leq \alpha\Bigr\}
    \label{eq:per_expert_lambda}
\end{equation}
and aggregate via majority vote at test time. These thresholds capture the scoring-scale heterogeneity identified in \S\ref{sec:avg_vote}, but $\lambda_j^{(0)}$ cannot serve as the final threshold: during calibration each expert is evaluated in isolation, while at test time a collective vote produces the prediction. This mismatch breaks exchangeability of the loss random variables required by CRC (Appendix~\ref{sec:appendix_remarks}).

Our insight is to use $\lambda_j^{(0)}$ as \textit{ratio initialization}, not as the final threshold. It estimates the approximate \textit{ratio} between per-expert thresholds at the target risk level. \S\ref{sec:joint} uses this ratio as a fixed shape to define a unified decision function~$C_t(x)$ that preserves exchangeability (Theorem~\ref{thm:crc-validity-mc3}).


\begin{algorithm}[t]
\small
\renewcommand{\baselinestretch}{1.2}\selectfont
\caption{MC$^3$}
\label{alg:mc3}
\begin{algorithmic}[1]
\Require \parbox[t]{\dimexpr\linewidth-2.5em\relax}{Calibration set $\mathcal{D}_{\mathrm{cal}} = \{(X_i, Y_i)\}_{i=1}^n$, per-expert scores $f_1, \ldots, f_K$, risk level $\alpha$, loss $\ell$ with upper bound $B$; per-expert decision $C_\lambda^{(j)}$ from Eq.~\ref{eq:decision}.}
\Statex
\Statex \textbf{Phase 1: Marginal Calibration} \hfill {\footnotesize\textit{$\triangleright$ per-expert scale}}
\For{$j = 1, \ldots, K$}
    \State $R_n^{(j)}(\lambda) \gets \tfrac{1}{n}\sum_{i=1}^n \ell(C_\lambda^{(j)}(X_i), Y_i)$
    \State $\lambda_j^{(0)} \gets \inf\{\lambda \geq 0 : R_n^{(j)}(\lambda) \leq \alpha\}$
\EndFor
\Statex \textbf{Phase 2: Joint Threshold Searching} \hfill {\footnotesize\textit{$\triangleright$ global scalar~$t$}}
\State Define $\lambda_j(t) \gets t \cdot \lambda_j^{(0)}$ for all $j$
\State $C_t(x) \gets \mathrm{MajVote}\bigl(C_{\lambda_1(t)}^{(1)}(x), \ldots, C_{\lambda_K(t)}^{(K)}(x)\bigr)$
\State $R_n(t) \gets \tfrac{1}{n}\sum_{i=1}^n \ell(C_t(X_i), Y_i)$
\State $\hat{t} \gets \inf\bigl\{t \geq 0 : \tfrac{n}{n+1} R_n(t) + \tfrac{B}{n+1} \leq \alpha\bigr\}$
\State \textbf{return} $\hat{t}$ and $\{\lambda_j = \hat{t} \cdot \lambda_j^{(0)}\}_{j=1}^K$
\Statex \textbf{Deployment.} On test $X_{n+1}$, output $C_{\hat{t}}(X_{n+1})$.
\Statex \hspace{1em} \textit{Guarantee} (Theorem~\ref{thm:crc-validity-mc3}): $\mathbb{E}[\ell(C_{\hat{t}}(X_{n+1}), Y_{n+1})] \leq \alpha$.
\end{algorithmic}
\end{algorithm}

\subsection{Joint Threshold Searching}
\label{sec:joint}

Using the ratios from \S\ref{sec:init} as a fixed shape, we introduce a single global scalar $t \geq 0$ that scales all per-expert thresholds proportionally:
\begin{equation}
    \lambda_j(t) = t \cdot \lambda_j^{(0)}
    \label{eq:ratio_scale}
\end{equation}
We define the \textit{joint vote decision function}, replacing $C_\lambda(x)$ in \S\ref{sec:formulation}:
\begin{equation}
    \scalebox{0.9}{$\displaystyle C_t(x) = \mathrm{MajVote}\!\bigl(C_{\lambda_1(t)}^{(1)}(x),\;\ldots,\;C_{\lambda_K(t)}^{(K)}(x)\bigr)$}
    \label{eq:joint_decision}
\end{equation}
Analogous to~$\lambda$ in Eq.~\ref{eq:decision}, $t$ encodes the level of conservativeness: as $t$ increases, all $\lambda_j(t)$ scale up proportionally, raising the abstention bar for each expert and producing more abstentions, so $\ell(C_t(X_i), Y_i)$ is non-increasing in~$t$ (Lemma~\ref{lem:lossmono}), satisfying the monotonicity requirement of CRC. Threshold searching is performed directly over~$t$:
\begin{equation}
    \hat{t} = \inf\Bigl\{t \geq 0 : \frac{n}{n{+}1}\,R_n(t) + \frac{B}{n{+}1} \leq \alpha\Bigr\}
    \label{eq:t_hat}
\end{equation}
yielding final per-expert thresholds $\lambda_j = \hat{t} \cdot \lambda_j^{(0)}$. Sharing a single scalar $t$ across experts is primarily a search-efficiency choice, reducing a $K$-dimensional joint search to a one-dimensional problem along the linear ray $\lambda_j(t) = t \cdot \lambda_j^{(0)}$, rather than the mechanism securing the CRC bound; the bound itself rests on the two conditions (P1) monotonicity and (P2) exchangeability. Concretely, because each calibration loss $L_i = \ell(C_t(X_i), Y_i)$ and the test loss $L_{n+1} = \ell(C_t(X_{n+1}), Y_{n+1})$ are computed by the same function applied to exchangeable data, the loss random variables $\{L_1, \ldots, L_{n+1}\}$ remain exchangeable, and the CRC bound applies (Theorem~\ref{thm:crc-validity-mc3}). Appendix~\ref{sec:appendix_proofs} shows formal proofs for multi-expert methods.

MC$^3$ simultaneously achieves three properties that no approach from \S\ref{sec:multi_expert} provides: (1)~risk guarantee via a shared $C_t$ across phases; (2)~per-expert heterogeneity, as $\lambda_j$ differ across experts; and (3)~continuous parameterization over the scalar~$t$. The procedure is summarized in Algorithm~\ref{alg:mc3}.

\section{Benchmark Construction}
\label{sec:datasets}

Evaluating CRC for open-ended dialogue requires a pairwise preference benchmark in which every candidate model exposes its logits, yet no such benchmark currently exists. We construct such a benchmark by sampling 100 dialogue contexts from each of three datasets spanning diverse conversational scenarios: ESConv~\cite{esconv} (emotional support), MSC~\cite{msc} (multi-session social chat), and DREAM~\cite{dream} (dialogue comprehension). For each context, candidate responses are generated by four open-weight LLMs: \texttt{gemma-3-12b-it}~\cite{gemma3}, \texttt{Mistral-Nemo-Instruct-2407}~\cite{mistral}, \texttt{Qwen2.5-7B-Instruct}~\cite{qwen25}, and \texttt{Llama-3.1-8B-Instruct}~\cite{llama3}. We enumerate all $\binom{4}{2} = 6$ response pairs per context, yielding 1{,}800 annotated pairs in total with human pairwise preference labels across five quality dimensions.\footnote{\textsc{Panel} is publicly available at \url{https://huggingface.co/datasets/EstellaCheng42/panel}.} The detailed annotation protocol is provided in Appendix~\ref{sec:appendix_data}. Regarding potential data contamination: although the dialogue contexts predate the judge models' pre-training cutoffs, both the candidate responses and the human preference labels are newly produced in this work, so the response--label pairs used for CRC calibration remain uncontaminated; our meta-evaluation of judge--human alignment therefore gains no shortcut from context memorization.

The four LLMs in \textsc{Panel} admit two complementary orderings: generation quality (measured by human preference win rates) and single-judge evaluation accuracy (measured against human pairwise labels). We report both in Appendix~\ref{sec:appendix_data} (Figs.~\ref{fig:winrate} and~\ref{fig:eval_acc}); they agree on the same model being strongest along both axes but otherwise diverge, motivating a multi-expert framework robust to individual-judge weakness. The generator ordering directly governs CRC calibration difficulty: well-separated generators (e.g., on ESConv) yield large pairwise margins that make calibration tractable, whereas clustered generators (e.g., on MSC) compress margins and demand more discriminative conformity scores.

\begin{table*}[t]
\centering
\caption{Main results across scoring functions and three datasets ($\alpha = 0.1$). \textit{Pointwise}: each response scored independently, pairwise signal derived from score difference. \textit{Pairwise}: a separate judge LLM directly compares two responses via logit extraction. \textit{Multi.}: multi-expert ensembles using homogeneous experts (same model, varied prompts/seeds). \textbf{AccR}: acceptance rate; \textbf{Acc}: pairwise accuracy; \textbf{AUC}: area under the ROC curve; \textbf{Risk}: fraction of all pairs committed to an incorrect preference. CRC requires Risk~$\leq \alpha$. Within each category and dataset, the best result is in \textbf{bold}. $^{\dagger}$Pref.\ Prob.\ w/ swap is the bidirectional preference scoring (BPE) of~\citet{scope}.}
\label{tab:main_results}
\small
\setlength{\tabcolsep}{4.5pt}
\begin{tabular}{@{}ll cccc cccc cccc@{}}
\toprule
& \multirow{2}{*}{\textbf{Score}} & \multicolumn{4}{c}{\textbf{ESConv}} & \multicolumn{4}{c}{\textbf{MSC}} & \multicolumn{4}{c}{\textbf{DREAM}} \\
\cmidrule(lr){3-6} \cmidrule(lr){7-10} \cmidrule(lr){11-14}
& & \textbf{AccR}$\uparrow$ & \textbf{Acc}$\uparrow$ & \textbf{AUC}$\uparrow$ & \textbf{Risk}$\downarrow$ & \textbf{AccR}$\uparrow$ & \textbf{Acc}$\uparrow$ & \textbf{AUC}$\uparrow$ & \textbf{Risk}$\downarrow$ & \textbf{AccR}$\uparrow$ & \textbf{Acc}$\uparrow$ & \textbf{AUC}$\uparrow$ & \textbf{Risk}$\downarrow$ \\
\midrule
\multirow{7}{*}{\rotatebox{90}{\scriptsize Pointwise}}
& Causal & \underline{.644} & \underline{.743} & \underline{.828} & .124 & .351 & .629 & .677 & .121 & \underline{.473} & .590 & .685 & .140 \\
& Consistency & .356 & .654 & .687 & .116 & .158 & .553 & .528 & \underline{.082} & .158 & .522 & .497 & \textbf{.069} \\
& Self-Certainty & .481 & .661 & .684 & .122 & \underline{.381} & .618 & \textbf{.753} & \textbf{.064} & .322 & .556 & .628 & \underline{.076} \\
& DeepConf & .480 & .668 & .624 & .185 & .309 & \underline{.667} & .690 & .094 & .255 & .532 & .565 & .096 \\
& Log Prob & \textbf{.794} & \textbf{.834} & \textbf{.910} & \underline{.091} & \textbf{.389} & .647 & .676 & .118 & \textbf{.547} & \textbf{.702} & \textbf{.738} & .135 \\
& \quad w/ self-judge & .388 & .643 & .684 & \textbf{.087} & .377 & \textbf{.688} & \underline{.693} & .097 & .396 & \underline{.620} & \underline{.687} & .100 \\
\midrule
\multirow{5}{*}{\rotatebox{90}{\scriptsize Pairwise}}
& Rubric & .314 & .619 & .713 & \textbf{.042} & .196 & .486 & .535 & \underline{.079} & .172 & .573 & .678 & \textbf{.036} \\
& Verb.\ Conf. & .747 & .769 & .880 & .107 & \underline{.390} & .641 & .699 & .108 & .447 & .662 & .791 & \underline{.084} \\
& Pref.\ Prob.  & .815 & \underline{.873} & \textbf{.937} & .094 & .346 & .654 & .692 & .103 & .463 & \underline{.720} & \underline{.796} & .093 \\
&\quad w/ swap$^{\dagger}$ & \underline{.847} & .854 & .913 & .097 & .348 & \underline{.663} & \textbf{.743} & \textbf{.074} & \underline{.479} & \textbf{.728} & \textbf{.809} & .094 \\
& \quad w/ self-judge & \textbf{.852} & \textbf{.876} & \underline{.936} & \underline{.092} & \textbf{.405} & \textbf{.672} & \underline{.736} & .095 & \textbf{.548} & .692 & .784 & .101 \\
\midrule
\multirow{2}{*}{\rotatebox{90}{\scriptsize Multi.}}
& Score Averaging & \underline{.951} & \underline{.882} & \underline{.940} & \textbf{.094} & \underline{.489} & \underline{.712} & \textbf{.800} & \textbf{.088} & \underline{.550} & \underline{.751} & \textbf{.830} & \underline{.085} \\
& Decision Voting & \textbf{.973} & \textbf{.887} & \textbf{.945} & \underline{.098} & \textbf{.495} & \textbf{.732} & \underline{.798} & \underline{.097} & \textbf{.566} & \textbf{.755} & \textbf{.830} & \textbf{.082} \\
\bottomrule
\end{tabular}
\end{table*}

\section{Experiments}
\label{sec:experiments}
\subsection{Experimental Setup}

\textbf{Setup.} We evaluate all scoring functions on each dataset and report results averaged over 5 random 1:1 calibration/test splits, partitioned by \textit{conversation ID} to prevent leakage between turns of the same conversation. The risk level is set to $\alpha = 0.1$ throughout for all methods. All multi-expert methods are based on pairwise logit-based preference unless otherwise noted. All experiments are conducted on NVIDIA A800 80GB GPUs. 

\textbf{Scoring Functions.} We evaluate pointwise scores (log probability, Self-Certainty~\cite{selfcertainty}, DeepConf~\cite{deepconf}, Causal~\cite{causal}, Consistency~\cite{consistency}) and pairwise scores (Pref.~Prob., Verbalized Confidence, Rubric). For pointwise scores, the pairwise conformity signal is the difference between the two response scores, with the sign indicating preference and the magnitude indicating confidence. Per-method implementation details are in Appendix~\ref{sec:appendix}.

\textbf{Evaluator Selection.} We use the top-$K$ strongest evaluators throughout: $K{=}1$ for single-judge experiments, and $K{=}3$ for multi-expert experiments, with the joint decision committed via standard majority voting. Evaluators are ranked by single-judge accuracy against human pairwise labels; detailed rankings are reported in Appendix~\ref{sec:appendix_data}. Rows marked ``w/ swap'' implement~\citet{scope}'s bidirectional preference scoring (BPE): the judge is queried twice with candidate order swapped, and the preference probabilities are averaged. In the multi-expert setting, a \textit{homogeneous} ensemble queries the same model $K$ times with varied random seeds and prompt templates at temperature 0.7; a \textit{heterogeneous} ensemble uses the top-$K$ distinct models.

\textbf{Metrics.} Within CRC, risk is held below $\alpha$ by abstaining on uncertain pairs, so the operative question is how much coverage a method retains under this guarantee. We report four metrics, all computed over non-tie pairs (human-annotated ties are excluded): \textit{Acceptance Rate} (AccR~$\uparrow$, the fraction of pairs not abstained), the coverage recovered under the risk constraint; \textit{Accuracy} (Acc~$\uparrow$, the fraction of pairs whose score-difference sign matches the human preference label, computed over all pairs regardless of CRC acceptance), which gauges the quality of the aggregated decision independently of abstention; \textit{AUC} ($\uparrow$, area under the ROC curve, capturing the threshold-independent discriminative power of the scoring function); and \textit{Risk} ($\downarrow$, the rate of wrong-and-accepted predictions, with abstentions incurring zero loss; CRC ensures Risk~$\leq \alpha$, with a value closer to $\alpha$ indicating tighter calibration), the quantity CRC provably controls.

\begin{table*}[t]
\centering
\caption{Multi-expert aggregation results across three datasets ($\alpha = 0.1$, 5 random 1:1 cal/test splits). \textbf{AccR}: acceptance rate; \textbf{Acc}: pairwise accuracy; \textbf{Risk}: fraction of all pairs committed to an incorrect preference. CRC requires Risk~$\leq \alpha$. The best result is in \textbf{bold}. $^{\dagger}$Test-time Voting breaks exchangeability between calibration and test, voiding its CRC guarantee, and is excluded from the ranking.}
\label{tab:multi_expert}
\small
\setlength{\tabcolsep}{7pt}
\begin{tabular}{@{}l c ccc ccc ccc@{}}
\toprule
 \multicolumn{1}{l}{\multirow{2}{*}{\textbf{Ensemble}}} & \multicolumn{1}{c}{\multirow{2}{*}{\textbf{Strategy}}}
  & \multicolumn{3}{c}{\textbf{ESConv}}
  & \multicolumn{3}{c}{\textbf{MSC}}
  & \multicolumn{3}{c}{\textbf{DREAM}} \\
\cmidrule(lr){3-5} \cmidrule(lr){6-8} \cmidrule(lr){9-11}
& & \textbf{AccR}$\uparrow$ & \textbf{Acc}$\uparrow$ & \textbf{Risk}$\downarrow$
    & \textbf{AccR}$\uparrow$ & \textbf{Acc}$\uparrow$ & \textbf{Risk}$\downarrow$
    & \textbf{AccR}$\uparrow$ & \textbf{Acc}$\uparrow$ & \textbf{Risk}$\downarrow$ \\
\midrule
\multirow{2}{*}{Homogeneous}
& Score Averaging     & \underline{.951} & \underline{.882} & \textbf{.094} & \underline{.489} & \underline{.712} & \textbf{.088} & \underline{.550} & \underline{.751} & \underline{.085} \\
& Decision Voting     & \textbf{.973} & \textbf{.887} & \underline{.098} & \textbf{.495} & \textbf{.732} & \underline{.097} & \textbf{.566} & \textbf{.755} & \textbf{.082} \\
\midrule
\multirow{4}{*}{Heterogeneous}
& Score Averaging     & .845 & .853 & \underline{.092} & \underline{.396} & \textbf{.685} & \underline{.086} & \underline{.508} & .715 & \underline{.089} \\
& Decision Voting     & \underline{.846} & \underline{.866} & \textbf{.090} & .364 & .677 & \textbf{.083} & .449 & \underline{.725} & .089 \\
& \textcolor{gray!70}{Test-time Voting$^{\dagger}$}  & \textcolor{gray!70}{.962} & \textcolor{gray!70}{.866} & \textcolor{gray!70}{.118} & \textcolor{gray!70}{.670} & \textcolor{gray!70}{.677} & \textcolor{gray!70}{.139} & \textcolor{gray!70}{.692} & \textcolor{gray!70}{.725} & \textcolor{gray!70}{.154} \\
& MC$^3$    & \textbf{.894} & \textbf{.866} & .092 & \textbf{.462} & \underline{.677} & \textbf{.083} & \textbf{.540} & \textbf{.725} & \textbf{.081} \\
\bottomrule
\end{tabular}
\end{table*}

\subsection{Results and Analysis}

\paragraph{Base score selection (Table~\ref{tab:main_results}).}
Table~\ref{tab:main_results} reports our empirical study of base scoring functions for CRC in open-ended dialogue. The base score is the most critical component of CRC: it determines whether the calibrated threshold reflects genuine preference uncertainty. Yet no prior work has systematically identified the most effective scoring functions. We address this gap by comparing two families: pointwise scores (e.g., Log~Prob, Self-Certainty, DeepConf), where each response is scored independently, and pairwise scores (e.g., Pref.~Prob., Verbalized Confidence, Rubric), where the judge directly compares two responses. Empirically, logit-based pairwise extraction (Pref.~Prob.) outperforms the strongest pointwise score (Log~Prob) in both accuracy and AUC across all three datasets, while keeping risk near~$\alpha$. This advantage reflects a deeper alignment between pairwise judgment and CRC: since open-ended dialogue lacks an absolute quality criterion, relative preference within a shared context provides a more natural conformity signal. We therefore adopt Pref.~Prob.\ as the per-expert base score for all multi-expert experiments.


\paragraph{Bias robustness (Table~\ref{tab:main_results}).}
Known biases in LLM-as-a-Judge evaluation have a diminishing impact with pairwise logit-based scoring. Self-judge bias, which is severe under pointwise scoring (Log~Prob drops by 19.1~pp on ESConv), becomes negligible in the pairwise setting: enabling self-judging changes accuracy by less than 3~pp across all datasets and in both directions, well within typical run-to-run variation. Position bias likewise yields only marginal gains under swap correction (around 1~pp on average), suggesting its impact is limited at current model scale. These surface biases cannot explain the broader gains of multi-expert aggregation, which addresses model-idiosyncratic tendencies that simple rotation alone cannot reach. Take-away: surface biases (position, self-enhancement) are largely resolved at current model scale under pairwise-logit scoring, yet the persistent gains from multi-expert aggregation in Table~\ref{tab:multi_expert} point to a deeper, model-specific bias that no single-expert correction can reach, motivating consensus prediction across experts.

\paragraph{Multi-expert in homogeneous settings (Tables~\ref{tab:main_results},~\ref{tab:multi_expert}).}
When all $K$ experts share the same model (varied prompts and seeds), Decision Voting improves over a single Pref.~Prob.\ judge on both acceptance rate and accuracy across all three datasets, most strongly on acceptance rate (up to $+15.8$~pp), while Score Averaging gives smaller, less consistent gains. Risk stays at or below~$\alpha$ throughout (at most $.098$), so the guarantee holds with no violation. The edge of Decision Voting reflects vote-level CRC acting directly on the decision function rather than on a composite score. This matches \S\ref{sec:avg_vote}: homogeneous experts share one scoring scale, so neither score-level nor threshold-level aggregation introduces a distributional mismatch.

\paragraph{The heterogeneous challenge (Table~\ref{tab:multi_expert}).}
The calibration strategy becomes decisive when experts come from different models. Both \textit{Score Averaging} and \textit{Decision Voting} stay risk-valid here, yet recover only limited coverage, for a shared reason: a single uniform threshold cannot match the experts' distinct scoring scales, so narrow-range experts are silenced and the ensemble abstains more often. \textit{Score Averaging} compounds this by first collapsing the incompatible per-expert scores into one composite. \textit{Test-time Voting} instead appears to recover high coverage, but it calibrates each expert in isolation and votes only at test time, so the decision function differs between calibration and test; exchangeability breaks, risk climbs above~$\alpha$ (up to $.154$), and we exclude it from the ranking. Take-away: under a single shared threshold, heterogeneous ensembles collapse toward their most wide-scale expert; recovering true multi-expert coverage requires a decision function that preserves per-expert scales while remaining shared across calibration and test, which is exactly what MC$^3$ delivers.

\paragraph{MC$^3$ in the heterogeneous setting (Table~\ref{tab:multi_expert}).}
Ratio initialization (\S\ref{sec:init}) captures each expert's native scale via $\lambda_j^{(0)}$, and joint threshold searching (\S\ref{sec:joint}) finds a global scalar $\hat{t}$ that proportionally rescales all per-expert thresholds through the unified decision function $C_t(x)$. As a result, MC$^3$ recovers the highest acceptance rate among the CRC-valid methods on all three datasets, while keeping risk within~$\alpha$. Its accuracy closely tracks that of the other voting methods, since all use the same majority vote and differ only in threshold selection; the methods separate on AccR and Risk, not on Acc, which scores every pair regardless of acceptance. The full benefit of multi-expert evaluation emerges only when calibration and aggregation use the same decision function. Aggregation should therefore act at the decision level with per-expert thresholds, rather than at the score level through a shared composite.

\section{Related Work}
\label{sec:related}

\paragraph{Conformal Prediction in LLMs.}
Conformal prediction has been increasingly applied to quantify uncertainty in LLM generation tasks. Early extensions apply conformal methods to multiple-choice QA~\cite{kumar2023conformal}, open-ended language modeling~\cite{quach2023conformal}, and API-based settings without logit access~\cite{su2024api}. Two recent works extend CRC to pairwise LLM-as-a-Judge evaluation: \citet{jung2024trust} calibrates a single LLM judge by sampling multiple simulated annotators from one model to provide a provable guarantee on human agreement, and \citet{scope} optimizes the conformity score function itself to improve selective pairwise judging. However, neither addresses multi-expert pairwise LLM-as-a-Judge CRC.


\paragraph{Preference Datasets for Open-Ended Dialogues.}
Preference datasets for LLM alignment have grown rapidly (e.g., HH-RLHF~\cite{hhrlhf}, UltraFeedback~\cite{ultrafeedback}), yet primarily target general instruction-following along a single quality dimension. Open-ended conversational scenarios remain underserved. Many dialogue-focused datasets rely on \textit{within-model} preference pairs and provide no human pairwise comparisons across different LLMs; ESC-Pro~\cite{Zhao2025} and EmPO~\cite{Sotolar2024} are representative cases. Among the few that attempt \textit{between-model} evaluation, HEART~\cite{heart2025} evaluates proprietary black-box models without logit access, and o2mDial~\cite{Lee2025} covers only simple chitchat via closed-source APIs. Neither supports white-box calibration. 


\section{Discussion}
Two design choices in our framework merit clarification. Our framework requires logit access, restricting applicability to open-weight LLMs. This aligns with how LLM judges are actually deployed in high-stakes, privacy-sensitive settings, e.g., clinical or enterprise pipelines under HIPAA/GDPR, where local open-weight judges such as Prometheus~\cite{kim2024prometheus} are the norm; API-only judges instead require alternative conformity scores~\cite{su2024api}. Our $0/1$ loss further treats abstention as free, an intentional fit for high-stakes settings where errors are intolerable and human deferral is a safe fallback, keeping the target risk directly interpretable as an exact error rate. Deferral cost is tracked transparently via Acceptance Rate (e.g., MC$^3$'s AccR of $0.462$ on MSC at $\alpha{=}0.1$ acknowledges a ${\sim}54\%$ deferral rate); extending to cost-sensitive settings only requires swapping the abstention cost in the loss.

\section{Conclusion}
We formulated pairwise LLM-as-a-Judge evaluation as selective prediction under CRC and conducted the first systematic study of multi-expert CRC. We designed two CRC-adapted strategies, \textit{Score Averaging} and \textit{Decision Voting}, that work well for homogeneous experts but recover only limited coverage for heterogeneous ones. To resolve this, we proposed MC$^3$, which uses per-expert threshold ratios under a unified decision function~$C_t(x)$ shared by calibration and test, preserving exchangeability. On \textsc{Panel}, our 1{,}800-pair human pairwise-preference benchmark with full logit access, MC$^3$ recovers the highest acceptance rate among CRC-valid methods across all three domains, at comparable accuracy and while maintaining the formal risk guarantee.

\section*{Acknowledgements}
\label{sec:ack}
This research was supported by the Australian Government through the Australian Research Council's Discovery Project funding scheme (Grant No.: DP260100218). Dr Qiuhong Ke is the recipient of an Australian Research Council Discovery Early Career Researcher Award (project number DE250100030) funded by the Australian Government.
\section*{Limitations}

Our framework has two open limitations. First, we address only the selective-prediction frontend (commit vs.\ abstain); the design of the downstream human-review workflow for abstained cases remains open. Second, the CRC guarantee relies on exchangeability, which can break under distribution shift. A practical mitigation is to monitor per-expert threshold ratios $\lambda_j^{(0)}$ on a rolling window and rerun Algorithm~\ref{alg:mc3} once drift exceeds a preset threshold; this recalibration is cheap because CRC is a training-free threshold search.

\section*{Ethics Statement}
All preference labels in \textsc{Panel} are produced
in-house by three annotators from our research team, without
crowdsourcing, following the multi-dimensional rubric in
Appendix~\ref{sec:appendix_data}; no
separate institutional ethics review was required, as annotation was
carried out internally over publicly released, non-identifying dialogue
data. All dialogue contexts are used consistent with their original
licenses: ESConv~\cite{esconv} (CC BY-NC 4.0), MSC~\cite{msc}
(CC BY-NC 4.0, Meta ParlAI), and DREAM~\cite{dream} (non-commercial
research use); none contain PII, and ESConv's emotionally sensitive
conversations are pseudonymous in the original release. The four
candidate response models are used under their public licenses:
\texttt{gemma-3-12b-it}~\cite{gemma3} (Gemma Terms of Use),
\texttt{Mistral-Nemo-Instruct-2407}~\cite{mistral} (Apache~2.0),
\texttt{Qwen2.5-7B-Instruct}~\cite{qwen25} (Apache~2.0), and
\texttt{Llama-3.1-8B-Instruct}~\cite{llama3} (Llama~3.1 Community
License). All response generation and judge evaluation run locally on
institutional GPUs, without transmitting dialogue content to
third-party APIs.

\bibliographystyle{acl_natbib}
\bibliography{custom}

\clearpage
\appendix

\section{Dataset and Annotation Details}
\label{sec:appendix_data}

\paragraph{Datasets.}
ESConv~\cite{esconv} contains dialogues between help-seekers and supporters discussing personal challenges, with emotionally sensitive exchanges that require empathy and strategic support. MSC~\cite{msc} consists of long-term, persona-grounded conversations spanning multiple sessions. DREAM~\cite{dream} is a dialogue-based reading-comprehension dataset comprising short, multi-turn exchanges drawn from English exam materials. Together, these datasets span three complementary registers: short comprehension-oriented exchanges (DREAM), extended persona-grounded social conversation (MSC), and emotionally complex support dialogues (ESConv).

\paragraph{Human Annotation.}
All pairwise preference labels are obtained via manual annotation by \textbf{three independent annotators}. For each response pair generated from the same dialogue context, each annotator independently determines which response is superior along five dimensions: \textit{relevance} to the immediate dialogue history, \textit{specificity} of the information provided, \textit{consistency} with conversational context and common sense, \textit{empathy} in understanding and responding to emotional states, and \textit{overall} quality. The \textit{overall} dimension serves as the final preference label in all subsequent experiments. Annotators may also indicate a \textit{Tie} when the two responses are of comparable quality. To mitigate position bias, the presentation order is randomized for each pair.

\paragraph{Adjudication Protocol.}
Each of the 1{,}800 pairs is labeled independently by all three annotators. Pairs on which all three annotators agree on the \textit{overall} dimension are assigned the consensus label directly. For pairs with at least one disagreement, the three annotators convene in a joint discussion session, review the pair together to surface their underlying reasoning, and reach a consensus label through deliberation. Adjudication is applied to the \textit{overall} dimension, which serves as the preference label in all experiments.

\paragraph{Model Quality Rankings.}
Figures~\ref{fig:winrate} and~\ref{fig:eval_acc} report the two complementary orderings of the four LLMs in \textsc{Panel}: human-annotated win rates (generation quality) and pairwise evaluation accuracy (judge quality), respectively.

\section{Experimental Details}
\label{sec:appendix}

\paragraph{Scoring Function Implementations.}
The pointwise scores we consider are log probability, Self-Certainty~\cite{selfcertainty}, and DeepConf~\cite{deepconf}; each requires only a single forward pass per response. The Causal score~\cite{causal} uses fine-tuned RoBERTa classifiers. Consistency~\cite{consistency} scoring uses 5 independent samples per response with temperature 0.7, clustered via DeBERTa-based NLI. For Pref.~Prob., we extract logits at the \emph{first generated token position} for the tokens \texttt{A} and \texttt{B} as encoded by each model's own tokenizer; we verified that \texttt{A} and \texttt{B} are single-token under all four model tokenizers given our prompt formatting. The conformity score in Eq.~\ref{eq:base_score} is then computed as a \emph{binary softmax restricted to these two logits}, deliberately excluding the rest of the vocabulary. A full-vocabulary softmax would otherwise absorb probability mass into tokens such as \texttt{The}, leading whitespace, or end-of-sequence that carry no preference content under our prompt template, diluting the comparative signal that motivates pairwise judging. The ``w/ swap'' variant of Pref.~Prob.\ is the \textit{bidirectional preference scoring} (BPE) introduced by~\citet{scope}: we query the judge twice with the candidate order swapped, extract the binary-softmax probabilities under both orderings via the same protocol above, and average them into a position-symmetric per-expert score. We therefore include~\citet{scope}'s BPE directly in our base-score comparison under the ``Pref.~Prob.\ w/ swap'' row of Table~\ref{tab:main_results}. Verbalized Confidence (Verb.\ Conf.) prompts the judge to output a numerical confidence score along with its preference, while Rubric prompts the judge to select from a 7-point preference scale.

\begin{figure*}[t]
\centering
\begin{minipage}[t]{0.48\linewidth}
    \centering
    \includegraphics[width=\linewidth]{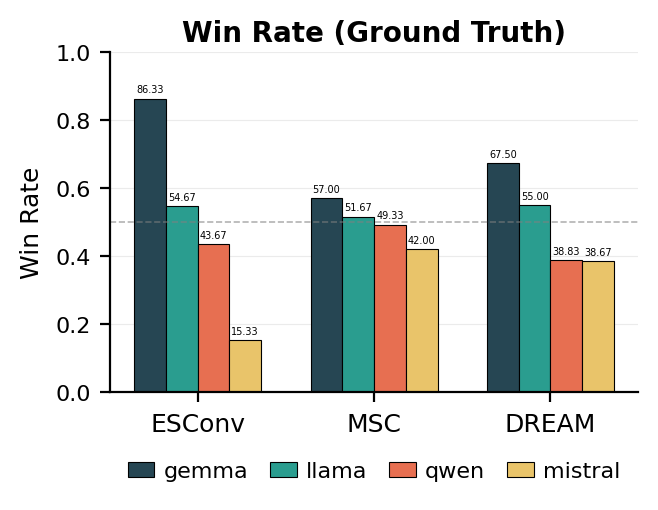}
    \caption{Human-annotated model win rates across the three dialogue datasets. The generator ranking is Gemma-3-12B $>$ Llama-3.1-8B $>$ Qwen2.5-7B $>$ Mistral-Nemo on every dataset.}
    \label{fig:winrate}
\end{minipage}\hfill
\begin{minipage}[t]{0.48\linewidth}
    \centering
    \includegraphics[width=\linewidth]{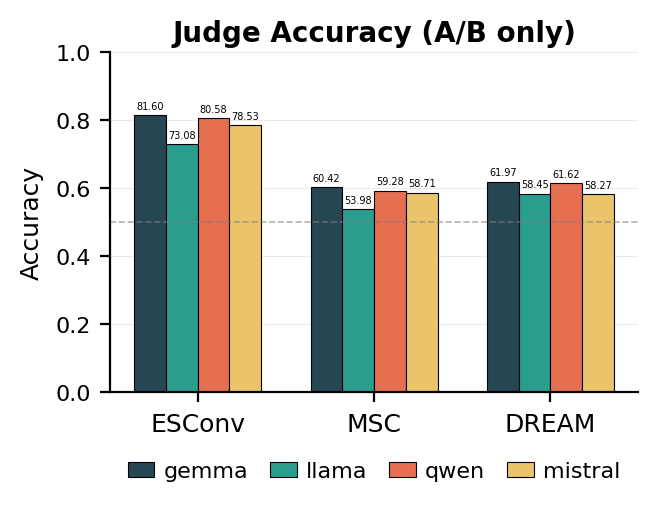}
    \caption{Single-judge accuracy of each LLM against human pairwise labels across the three dialogue datasets. The dataset-averaged ranking is Gemma $>$ Qwen $>$ Mistral $>$ Llama.}
    \label{fig:eval_acc}
\end{minipage}
\end{figure*}

\section{Proofs of CRC Validity for Multi-Expert Methods}
\label{sec:appendix_proofs}

This appendix establishes that the three multi-expert methods of \S\ref{sec:multi_expert} and \S\ref{sec:mc3}, namely \textit{Score Averaging}, \textit{Decision Voting}, and \textit{MC}$^3$, each satisfy the conditions of the Conformal Risk Control (CRC) theorem of \citet{crc} and therefore inherit its finite-sample, distribution-free risk guarantee
\begin{equation}
    \mathbb{E}\bigl[\ell(C_{\hat{\lambda}}(X_{n+1}), Y_{n+1})\bigr] \;\leq\; \alpha
    \label{eq:appendix_target}
\end{equation}
for an exchangeable sequence of calibration and test pairs $(X_i, Y_i)_{i=1}^{n+1}$. Throughout, $\ell(\hat{y}, y) := \mathbf{1}[\hat{y} \in \mathcal{Y} \wedge \hat{y} \neq y]$ is the indicator loss of Eq.~\ref{eq:loss}, so $\ell \in \{0, 1\}$ and the loss upper bound is $B = 1$.

\subsection{Proof Goal and Strategy}
\label{sec:appendix_goal}

\paragraph{Goal.}
For each of the three methods, with its associated calibrated threshold (denoted $\hat{\lambda}$ for Score Averaging and Decision Voting, and $\hat{t}$ for MC$^3$), we wish to verify Eq.~\ref{eq:appendix_target}.

\paragraph{Strategy.}
We invoke Theorem~1 of \citet{crc}, which we restate here for completeness:

\begin{theorem}[CRC; \citealt{crc}, Theorem~1]
\label{thm:crc}
Let $\{L_i(\cdot)\}_{i=1}^{n+1}$ be an exchangeable collection of random functions $L_i : \Lambda \to (-\infty, B]$ satisfying, almost surely, $L_i(\lambda_{\max}) \leq \alpha$ and $L_i(\lambda)$ non-increasing and right-continuous in $\lambda$. Define
\begin{equation*}
    \hat{\lambda} \;=\; \inf\!\left\{\,\lambda \,:\, \tfrac{n}{n{+}1}\widehat{R}_n(\lambda) + \tfrac{B}{n{+}1} \,\leq\, \alpha\right\},
\end{equation*}
where $\widehat{R}_n(\lambda) = \tfrac{1}{n}\sum_{i=1}^n L_i(\lambda)$. Then $\mathbb{E}[L_{n+1}(\hat{\lambda})] \leq \alpha$.
\end{theorem}

It suffices to verify two properties for each method:

\textbf{(P1) Monotonicity.} For every sample $(X_i, Y_i)$, the function $\lambda \mapsto L_i(\lambda)$ (or $t \mapsto L_i(t)$ for MC$^3$) is non-increasing.

\textbf{(P2) Exchangeability.} The collection $\{L_i(\cdot)\}_{i=1}^{n+1}$ is exchangeable as a sequence of random functions.
Boundedness ($L_i \in [0, 1]$) and right-continuity (each $L_i$ is a step function in its threshold variable) are immediate. Existence of $\lambda_{\max}$ with $L_i(\lambda_{\max}) = 0$ follows because for any threshold strictly exceeding $\max_j |f_j(X_i)|$ (resp. $\max_j |f_j(X_i)| / \lambda_j^{(0)}$ for MC$^3$, assuming $\lambda_j^{(0)} > 0$), every expert abstains and the loss vanishes. We can then take $\lambda_{\max} = \infty$ (or any finite upper bound satisfying the same).

We establish (P1) via the voting lemmas of \S\ref{sec:appendix_lemmas}, and (P2) via the following one-line observation, used uniformly across all three methods:

\begin{lemma}[Pointwise exchangeability]
\label{lem:exch}
Let $\{Z_i\}_{i=1}^{n+1}$ be an exchangeable sequence of random variables and $g$ a deterministic measurable function. Then $\{g(Z_i)\}_{i=1}^{n+1}$ is exchangeable.
\end{lemma}

\begin{proof}
For any permutation $\pi$ of $\{1, \ldots, n+1\}$, $(g(Z_{\pi(1)}), \ldots, g(Z_{\pi(n+1)})) = g_{*}(Z_{\pi(1)}, \ldots, Z_{\pi(n+1)}) \overset{d}{=} g_{*}(Z_1, \ldots, Z_{n+1}) = (g(Z_1), \ldots, g(Z_{n+1}))$, where $g_{*}$ denotes pointwise application of $g$.
\end{proof}

\subsection{Voting Lemmas}
\label{sec:appendix_lemmas}

Throughout this subsection, let $K$ be the number of experts and $M = \lfloor K/2 \rfloor + 1$ the strict-majority threshold, satisfying $K - M < M$.

\begin{lemma}[Per-expert and vote-count monotonicity]
\label{lem:perexpert}\label{lem:counts}
For each expert $j$, the per-expert decision $C_{\lambda}^{(j)}(x)$ defined in Eq.~\ref{eq:decision} satisfies, for any $\lambda_1 \leq \lambda_2$:
\begin{align*}
C_{\lambda_1}^{(j)}(x) = A      &\;\Longrightarrow\; C_{\lambda_2}^{(j)}(x) \in \{A, \text{abstain}\},\\
C_{\lambda_1}^{(j)}(x) = B      &\;\Longrightarrow\; C_{\lambda_2}^{(j)}(x) \in \{B, \text{abstain}\},\\
C_{\lambda_1}^{(j)}(x) = \text{abstain}      &\;\Longrightarrow\; C_{\lambda_2}^{(j)}(x) = \text{abstain}.
\end{align*}
Consequently, the vote counts $v_A(\lambda) := |\{j : C_\lambda^{(j)}(x) = A\}|$, $v_B(\lambda) := |\{j : C_\lambda^{(j)}(x) = B\}|$, and $v_\perp(\lambda) := |\{j : C_\lambda^{(j)}(x) = \text{abstain}\}|$ (with $v_A + v_B + v_\perp = K$) satisfy: $\lambda \mapsto v_A(\lambda)$ and $\lambda \mapsto v_B(\lambda)$ are non-increasing, while $\lambda \mapsto v_\perp(\lambda)$ is non-decreasing.
\end{lemma}

\begin{proof}
By Eq.~\ref{eq:decision}, $C_{\lambda}^{(j)}(x) = A \iff f_j(x) > \lambda$ and $C_{\lambda}^{(j)}(x) = B \iff f_j(x) < -\lambda$; as $\lambda$ increases, both events can only turn from true to false and abstention is absorbing, giving the per-expert containment. Each expert transition is therefore one of $A \to A$, $A \to \text{abstain}$, $B \to B$, $B \to \text{abstain}$, or $\text{abstain} \to \text{abstain}$, so $v_A$ and $v_B$ can only decrease and $v_\perp$ can only increase.
\end{proof}

\begin{lemma}[No vote-flip]
\label{lem:noflip}
For any $\lambda_1 \leq \lambda_2$, the majority-vote decision $C_{\lambda}^{\mathrm{vote}}(x) = \mathrm{MajVote}(C_\lambda^{(1)}(x), \ldots, C_\lambda^{(K)}(x))$ satisfies the same containment as Lemma~\ref{lem:perexpert} (with $C^{(j)}$ replaced by $C^{\mathrm{vote}}$).
\end{lemma}

\begin{proof}
We treat the three cases, using throughout the identity $v_A + v_B + v_\perp = K$ from Lemma~\ref{lem:counts}.

\textbf{Case 1:} $C^{\mathrm{vote}}_{\lambda_1}(x) = A$. Then $v_A(\lambda_1) \geq M$, so $v_B(\lambda_1) \leq K - v_A(\lambda_1) \leq K - M$. By Lemma~\ref{lem:counts}, $v_B(\lambda_2) \leq v_B(\lambda_1) \leq K - M < M$, so the vote at $\lambda_2$ cannot produce $B$. Therefore $C^{\mathrm{vote}}_{\lambda_2}(x) \in \{A, \text{abstain}\}$.

\textbf{Case 2:} $C^{\mathrm{vote}}_{\lambda_1}(x) = B$. By the symmetric argument, $C^{\mathrm{vote}}_{\lambda_2}(x) \in \{B, \text{abstain}\}$.

\textbf{Case 3:} $C^{\mathrm{vote}}_{\lambda_1}(x) = \text{abstain}$. Then $v_A(\lambda_1) < M$ and $v_B(\lambda_1) < M$. By Lemma~\ref{lem:counts}, $v_A(\lambda_2), v_B(\lambda_2) < M$. So $C^{\mathrm{vote}}_{\lambda_2}(x) = \text{abstain}$.
\end{proof}

\begin{lemma}[Loss monotonicity from no-flip]
\label{lem:lossmono}
Let $\lambda \mapsto C_\lambda(x)$ be any decision rule taking values in $\{A, B, \text{abstain}\}$ that satisfies the no-flip property of Lemma~\ref{lem:noflip} (equivalently, of Lemma~\ref{lem:perexpert}). Then for any $(X, Y)$, the loss $\lambda \mapsto \ell(C_\lambda(X), Y)$ is non-increasing in $\lambda$.
\end{lemma}

\begin{proof}
Fix $\lambda_1 \leq \lambda_2$. If $\ell(C_{\lambda_1}(X), Y) = 0$, then $C_{\lambda_1}(X) \in \{Y, \text{abstain}\}$; by the no-flip property, $C_{\lambda_2}(X) \in \{C_{\lambda_1}(X), \text{abstain}\} \subseteq \{Y, \text{abstain}\}$, so $\ell(C_{\lambda_2}(X), Y) = 0$. If $\ell(C_{\lambda_1}(X), Y) = 1$, then $\ell(C_{\lambda_2}(X), Y) \in \{0, 1\} \leq 1$.
\end{proof}

\subsection{CRC Validity for Score Averaging}
\label{sec:appendix_avg}

For Score Averaging, the scoring function $f_{\mathrm{avg}}(x) = \tfrac{1}{K}\sum_{j=1}^{K} f_j(x)$ is a deterministic function of the data-independent $\{f_j\}_{j=1}^K$. Standard CRC is applied to $f_{\mathrm{avg}}$, giving $L_i(\lambda) = \ell(C_\lambda(X_i), Y_i)$ with $C_\lambda$ from Eq.~\ref{eq:decision}.

\textbf{(P1)} Lemma~\ref{lem:perexpert}, applied with $f_{\mathrm{avg}}$ in place of $f_j$, gives the no-flip property for $\lambda \mapsto C_\lambda(X_i)$; Lemma~\ref{lem:lossmono} then yields monotonicity of $L_i$ in $\lambda$.

\textbf{(P2)} At any fixed $\lambda$, the map $(X_i, Y_i) \mapsto L_i(\lambda)$ is a single deterministic function; Lemma~\ref{lem:exch} gives exchangeability.

Both conditions of Theorem~\ref{thm:crc} hold; the CRC guarantee \eqref{eq:appendix_target} applies.

\subsection{CRC Validity for Decision Voting}
\label{sec:appendix_dvote}

Decision Voting uses a single shared threshold $\lambda$ and aggregates via $C_\lambda^{\mathrm{vote}} = \mathrm{MajVote}(C_\lambda^{(1)}, \ldots, C_\lambda^{(K)})$, with $\{f_j\}$ data-independent. The loss is $L_i(\lambda) = \ell(C_\lambda^{\mathrm{vote}}(X_i), Y_i)$.

\textbf{(P1)} Lemma~\ref{lem:noflip} gives the no-flip property for $\lambda \mapsto C_\lambda^{\mathrm{vote}}(X_i)$; Lemma~\ref{lem:lossmono} yields monotonicity of $L_i$ in $\lambda$.

\textbf{(P2)} At any fixed $\lambda$, $C_\lambda^{\mathrm{vote}}$ is a deterministic function of $\{f_j\}_{j=1}^K$, so $(X_i, Y_i) \mapsto L_i(\lambda)$ is a single deterministic function; Lemma~\ref{lem:exch} gives exchangeability.

Both conditions of Theorem~\ref{thm:crc} hold; the CRC guarantee \eqref{eq:appendix_target} applies.

\subsection{CRC Validity for MC\texorpdfstring{$^3$}{\textthreesuperior}}
\label{sec:appendix_mc3}

MC$^3$ uses Phase~1 to compute per-expert ratios $\{\lambda_j^{(0)}\}_{j=1}^K \geq 0$ (Eq.~\ref{eq:per_expert_lambda}). Phase~2 introduces a global scalar $t \geq 0$ and defines $C_t(x) = \mathrm{MajVote}(C_{\lambda_1(t)}^{(1)}(x), \ldots, C_{\lambda_K(t)}^{(K)}(x))$ with $\lambda_j(t) = t \cdot \lambda_j^{(0)}$ (Eq.~\ref{eq:ratio_scale}). The loss is $L_i(t) = \ell(C_t(X_i), Y_i)$.

\begin{theorem}[CRC validity for MC$^3$]
\label{thm:crc-validity-mc3}
Under the exchangeability of $(X_i, Y_i)_{i=1}^{n+1}$, the MC$^3$ threshold $\hat{t}$ defined in Eq.~\ref{eq:t_hat} satisfies $\mathbb{E}[\ell(C_{\hat{t}}(X_{n+1}), Y_{n+1})] \leq \alpha$.
\end{theorem}

\begin{proof}
\textbf{(P1)} Each $\lambda_j(t) = t \cdot \lambda_j^{(0)}$ is non-decreasing in $t$ (since $\lambda_j^{(0)} \geq 0$), so Lemma~\ref{lem:perexpert} applies to each expert's threshold trajectory. Lemma~\ref{lem:noflip} then gives the no-flip property for $t \mapsto C_t(X_i)$; Lemma~\ref{lem:lossmono} yields monotonicity of $L_i(t)$ in $t$.

\textbf{(P2)} Phase~1 produces a vector of per-expert ratios $(\lambda_1^{(0)}, \ldots, \lambda_K^{(0)})$ that fixes the shape of the decision function. At any fixed $t \geq 0$ and given these ratios, the joint decision function
\begin{equation*}
    C_t(x) = \mathrm{MajVote}\!\bigl(C_{t \lambda_1^{(0)}}^{(1)}(x), \ldots, C_{t \lambda_K^{(0)}}^{(K)}(x)\bigr)
\end{equation*}
is a deterministic function of $x$, so the map $(X_i, Y_i) \mapsto L_i(t) = \ell(C_t(X_i), Y_i)$ is a single deterministic function. Lemma~\ref{lem:exch} gives exchangeability of $\{L_i(t)\}_{i=1}^{n+1}$.

\textbf{Conclusion.} Both (P1) and (P2) hold; the boundedness, right-continuity, and $\lambda_{\max}$ conditions are addressed in \S\ref{sec:appendix_goal}. Applying Theorem~\ref{thm:crc} with $\hat{t}$ as in Eq.~\ref{eq:t_hat} yields $\mathbb{E}[L_{n+1}(\hat{t})] \leq \alpha$.
\end{proof}

\subsection{Remarks}
\label{sec:appendix_remarks}

\paragraph{Why MC$^3$ preserves exchangeability while \textit{Test-time Voting} does not.}
In MC$^3$, the per-expert thresholds $\lambda_j(t) = t \cdot \lambda_j^{(0)}$ define a single joint decision function $C_t$ that is \emph{used identically in calibration and at test}. By contrast, the \textit{Test-time Voting} baseline of \S\ref{sec:init} calibrates each expert independently at level $\alpha$ to obtain $\lambda_j^{(0)}$, and aggregates them through a majority vote only at test time. The mismatch between the calibration-time decision (per-expert independent) and the test-time decision (collective vote) breaks the symmetry required by (P2), and the CRC bound no longer applies.

\paragraph{Generality.}
The proofs above use only the structure of the per-expert decisions in Eq.~\ref{eq:decision} and the majority-vote aggregation. They do not depend on the specific choice of conformity score $f_j$ (any score satisfying the structure of Eq.~\ref{eq:decision} suffices), nor on the dimension or number of experts $K$, nor on the dataset.

\end{document}